\documentclass{article}

\usepackage{microtype}
\usepackage{graphicx}
\usepackage{booktabs} 

\usepackage{hyperref}

\usepackage[accepted]{icml2019}

\usepackage{bbm}
\usepackage{dsfont}
\usepackage{amssymb}
\usepackage{amsthm}
\usepackage{amsmath}
\usepackage{caption}
\usepackage{subcaption}
\newtheorem{theorem}{Theorem}

\newtheorem{proposition}{Proposition}
\newtheorem{lemma}{Lemma}
\newcommand{\defeq}{\mathrel{\mathop:}=}
\newcommand{\by}{\vec{y}}
\newcommand{\bx}{\vec{x}}
\newcommand{\bepsilon}{\vec{\epsilon}}

\icmltitlerunning{Data Poisoning Attacks on Stochastic Bandits}

\begin{document}

\twocolumn[
\icmltitle{Data Poisoning Attacks on Stochastic Bandits}



\icmlsetsymbol{equal}{*}

\begin{icmlauthorlist}
\icmlauthor{Fang Liu}{ece}
\icmlauthor{Ness Shroff}{ece,cse}
\end{icmlauthorlist}

\icmlaffiliation{ece}{Department of Electrical and Computer Engineering,}
\icmlaffiliation{cse}{Department of Computer Science and Engineering, The Ohio State University, Columbus, Ohio, USA}

\icmlcorrespondingauthor{Fang Liu}{liu.3977@osu.edu}
\icmlcorrespondingauthor{Ness Shroff}{shroff.11@osu.edu}

\icmlkeywords{Adversarial Machine Learning, Data Poisoning, Multi-armed Bandits}

\vskip 0.3in
]



\printAffiliationsAndNotice{}  

\begin{abstract}
Stochastic multi-armed bandits form a class of online learning problems that have important applications in online recommendation systems, adaptive medical treatment, and many others. Even though potential attacks against these learning algorithms may hijack their behavior, causing catastrophic loss in real-world applications, little is known about adversarial attacks on bandit algorithms. In this paper, we propose a framework of offline attacks on bandit algorithms and study convex optimization based attacks on several popular bandit algorithms. We show that the attacker can force the bandit algorithm to pull a target arm with high probability by a slight manipulation of the rewards in the data. Then we study a form of online attacks on bandit algorithms and propose an adaptive attack strategy against any bandit algorithm \emph{without the knowledge of the bandit algorithm}. Our adaptive attack strategy can hijack the behavior of the bandit algorithm to suffer a linear regret with only a logarithmic cost to the attacker. Our results demonstrate a significant security threat to stochastic bandits. 
\end{abstract}

\section{Introduction}
Understanding adversarial attacks on machine learning systems is essential to developing effective defense mechanisms and an important step toward trustworthy artificial intelligence. 
A class of adversarial attacks on machine learning that have received much attention is data poisoning~\cite{biggio2012poisoning,mei2015using,xiao2015feature,alfeld2016data,li2016data,wang2018data}. Here, the attacker is able to access the leaner's training data, and has the power to manipulate a fraction of the training data in order to make the learner satisfy certain objectives. This is motivated by modern industrial scale applications of machine learning systems, where data collection and policy updates are done in a distributed way. Attacks can happen when the learner is not aware of the attacker's access to the training data.

While there has been a line of research on adversarial attacks on deep learning~\cite{43405,huang2017adversarial,lin2017tactics} and supervised learning~\cite{biggio2012poisoning,mei2015using,xiao2015feature,alfeld2016data,li2016data}, little is known on adversarial attacks on stochastic multi-armed bandits, which is a form of online learning with limited feedback. This is potentially hazardous since stochastic MAB are widely used in the industry to recommend news articles~\cite{li2010contextual}, display advertisements~\cite{chapelle2015simple}, allocate medical treatment~\cite{thompson1933likelihood} among many others. Hence, understanding the impact of adversarial attacks on bandit algorithms is an urgent yet open problem.

Recently, there has been an important piece of offline data poisoning attacks for the contextual bandit algorithm~\cite{ma2018data}. They assume that the bandit algorithm updates periodically and that the attacker can manipulate the rewards in the data before the updates in order to hijack the behavior of the bandit algorithm. Consider the news recommendation as a running example for this offline attack model. A news website has $K$ articles (i.e., arms or actions) and runs a bandit algorithm to learn a recommendation policy. Every time a user visits the website, the bandit algorithm displays an article based on historical data. Then the website receives a binary reward indicating whether the user clicks on the displayed article or not. The website keeps serving the users throughout the day and updates the bandit algorithm during the night. An attacker can perform offline data poisoning attacks before the bandit algorithm is updated. More specifically, the attacker may poison the rewards collected during the daytime and control the behavior of the bandit algorithm as it wants. The authors in~\cite{ma2018data} show that the offline attack strategy on LinUCB-type contextual bandit algorithm~\cite{li2010contextual} can be formulated as a convex optimization problem. However, offline attack strategies for classical bandit algorithms are still open.

In contrast to offline attacks, an online form of data poisoning attacks against bandit algorithms has also been proposed by~\cite{jun2018adversarial}. They assume that the bandit algorithm updates step by step and the attacker sits in-between the bandit environment and the bandit algorithm. At each time step, the bandit algorithm pulls an arm and the attacker eavesdrops on the decision. Then the attacker makes an attack by manipulating the reward generated from the bandit environment. The bandit algorithm receives the poisoned reward without knowing the presence of the attacker and updates accordingly. The goal of the attacker is to control which arm appears to the bandit algorithm as the best arm at the end. Efficient attack strategies against $\epsilon$-greedy and Upper Confidence Bounds (UCB) algorithms are proposed by~\cite{jun2018adversarial}. However, online attack strategies for other bandit algorithms (e.g., Thompson Sampling~\cite{thompson1933likelihood} and UCBoost~\cite{liu2018ucboost}) are still unknown.

In this work, we have a systematic investigation of data poisoning attacks against bandit algorithms and address the aforementioned open problems. We study the data poisoning attacks in both the offline and the online cases. In the offline setting, the bandit algorithm updates periodically and the attacker can manipulate the rewards in the collected data before the update occurs. In the online setting, the attacker eavesdrops on the bandit algorithm and manipulates the feedback reward. The goal of the attacker is that the bandit algorithm considers the target arm as the optimal arm at the end. Specifically, we make the following contributions to data poisoning attacks on stochastic bandits.
\begin{enumerate}
\item We propose an optimization based framework for offline attacks on bandit algorithms. Then, we instantiate three offline attack strategies against $\epsilon$-greedy, UCB algorithm and Thompson Sampling, which are the solutions of the corresponding convex optimization problems. That is, there exist efficient attack strategies for the attacker against these popular bandit algorithms. 
\item We study the online attacks on bandit algorithms and propose an adaptive attack strategy that can hijack any bandit algorithm \emph{without knowing the bandit algorithm}. As far as we know, this is the first negative result showing that there is no robust and good stochastic bandit algorithm that can survive online poisoning attacks.
\item We evaluate our attack strategies by numerical results. Our attack strategies are efficient in forcing the bandit algorithms to pull a target arm at a relatively small cost. Our results expose a significant security threat as bandit algorithms are widely employed in the real world applications.
\end{enumerate} 

More recently, there is a line of works by~\cite{lykouris2018stochastic,gupta2019better} studying an adversarial corruption model. While they assume that the attacker has to attack all the arms before the learner chooses an arm, we consider the case where the attacker's strategy is aware of and adaptive to the decision of the learner. The difference leads to opposite  results. While they propose a robust bandit algorithm, our work shows that the existing algorithms are vulnerable to the adversarial attacks.

This paper is organized as the following. In Section~\ref{sec:problem}, we introduce the problem setting and attack system models. Then, we propose the offline attack problem in Section~\ref{sec:offline}. In Section~\ref{sec:online}, we study the online attack problem and propose an adaptive attack strategy. The numerical results to evaluate our attack strategies are provided in Section~\ref{sec:simul}. Finally, we conclude the paper with discussion in Section~\ref{sec:conclusion}.

\section{Problem Formulation}\label{sec:problem}
We consider the classical stochastic bandit setting. Suppose that there is a set $\mathcal{A} = \{1,2,\ldots,K\}$ of $K$ arms and the bandit algorithm proceeds in discrete time $t=1,2,\ldots, T$. At each round $t$, the algorithm pulls an arm $a_t\in\mathcal{A}$ and the bandit environment generates a reward $r_t\in\mathcal{R}$ such that 
\begin{equation}\label{eqn:reward}
r_t=\mu_{a_t} + \eta_t,
\end{equation}
where $\eta_t$ is a zero-mean, $\sigma$-subGaussian noise and $\mu_{a_t}$ is the instantaneous reward at time $t$. In other words, the expected reward of pulling arm $a$ is $\mu_a$. Note that $\{\mu_a\}_{a\in\mathcal{A}}$ are \emph{unknown to the bandit algorithm and the attacker}. 

The performance of the bandit algorithm is measured by the regret, which is the expected difference between the total rewards obtained by an oracle that always pulls the optimal arm (i.e., the arm with the largest expected reward $\max_{a\in\mathcal{A}}\mu_a$) and the accumulative rewards collected by the bandit algorithm up to time horizon $T$. Formally, the regret $R(T)$ is given by
\begin{equation}
R(T) = \mathbb{E}\left[\max_{a\in\mathcal{A}}\mu_a T - \sum_{t=1}^T r_t\right].
\end{equation}
In this work, we consider the uniformly good bandit algorithm that incurs sub-linear regret, i.e., $R(T) = o(T)$. 

For each reward $r_t$ returned from the bandit environment, the attacker manipulates the reward into 
\begin{equation}\label{eqn:poison}
r'_t = r_t+\epsilon_t.
\end{equation}
Then the accumulated attack cost of the attacker, $C(T)$, is measured by the norm of the vector $(\epsilon_1,\ldots,\epsilon_T)^T$. For simplicity, we consider the $l^2$-norm in the offline setting and the $l^1$-norm in the online setting. Note that the results in this work can be easily extended to any norm. For example, consider the $l^p$-norm, the total attack cost of the attacker is
\begin{equation}
C(T) = \left(\sum_{t=1}^T |\epsilon_t|^p\right)^{1/p}.
\end{equation}

Without loss of generality, we assume that arm $a^*$ is a suboptimal attack target, such that $\mu_{a^*} < \max_{a\in\mathcal{A}}\mu_a$. The attacker's goal is to manipulate the bandit algorithm into pulling arm $a^*$ frequently. To avoid being detected, the attacker also wants to keep the cost as small as possible. 

\subsection{Offline attack system model}
The offline attack system model is illustrated in Figure~\ref{fig:offlineModel}. Besides the bandit algorithm, the bandit environment and the attacker, there is a data buffer of size $T$. This models the setting where updates of the bandit algorithm happen in mini-batches of size $T$. For round $t= 1,\ldots,T$, the bandit algorithm sequentially pulls arm $a_t$. Then the environment generates the reward $r_t$ according to Equation~(\ref{eqn:reward}) and send the tuple $(a_t,r_t)$ to the data buffer. The data buffer stores the data stream until round $T$. At the end of round $T$, the attacker accesses the data buffer and poisons the data by Equation~(\ref{eqn:poison}). Finally, the data buffer sends the poisoned data stream $\{(a_t,r'_t)\}_{t\leq T}$ to the bandit algorithm and the bandit algorithm updates according to the received data without knowing the existence of the attacker.

The goal of the attacker in the offline setting is to force the bandit algorithm to pull the target arm $a^*$ \emph{with high probability} at round $T+1$ (i.e., after updating with poisoned data) while incurring only a small cost. This means that the attacker wants to make the poisoning effort $\epsilon_t$ as small as possible to keep stealthy.
\begin{figure}[t]
  \centering
    \includegraphics[width=0.3\textwidth]{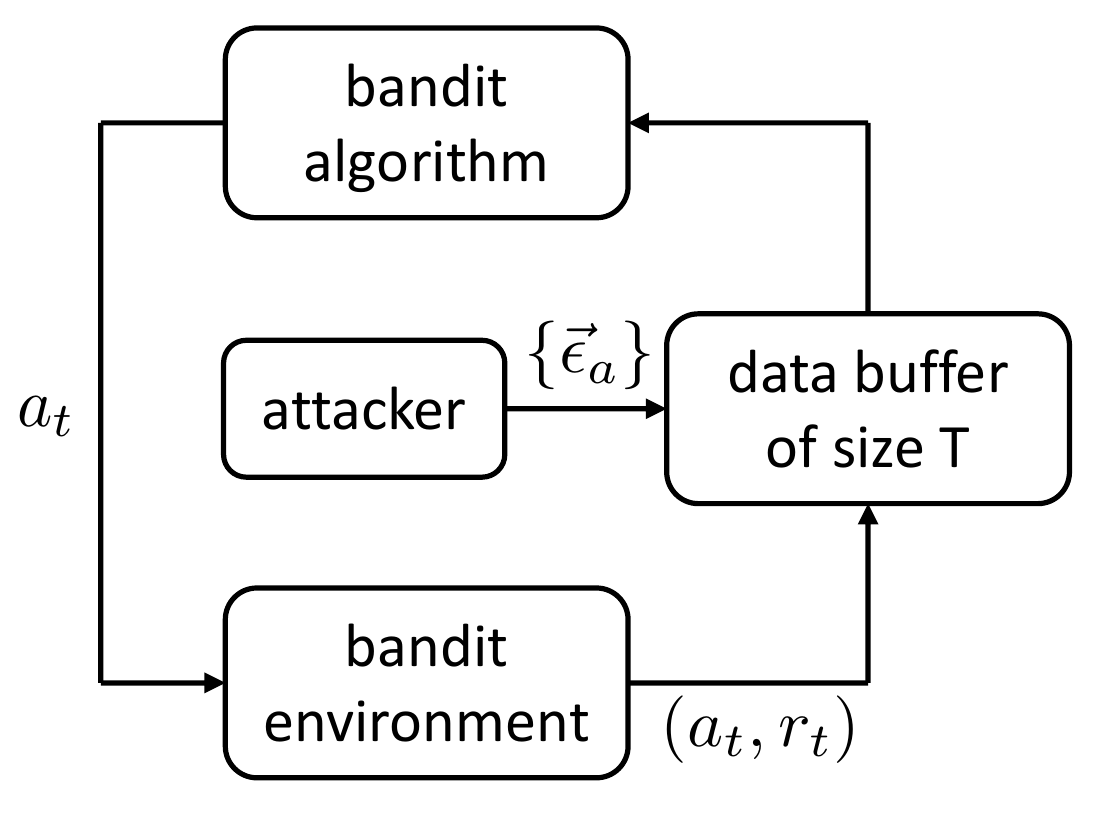}
     \caption{Offline attack system model}
     \label{fig:offlineModel}
\end{figure}

\subsection{Online attack system model}
The online attack system model is illustrated in Figure~\ref{fig:onlineModel}. In the online setting, the bandit algorithm updates instantly for each time step. The attacker stealthily monitors the decision of the bandit algorithm $a_t$ at each time $t$ and poisons the reward signal returned from the bandit environment by equation (\ref{eqn:poison}). Then the bandit algorithm receives the manipulated reward signal $r'_t$ and updates unaware of the attacker.

The goal of the attacker in the online setting is to hijack the behavior of the bandit algorithm \emph{with high probability} by manipulating the reward signal so that the bandit algorithm pulls the target arm $a^*$ in $\Theta(T)$ time steps. In the meantime, the attacker wants to control its attack cost by poisoning as infrequently as possible in order to avoid being detected. Note that $\epsilon_t=0$ is considered as no attack. 

By the linearity of expectation and $\eta_t$ is a zero-mean noise,
\begin{equation}\label{eqn:regretDecompose}
R(T) = \sum_{a\in\mathcal{A}} \left(\max_{i\in\mathcal{A}}\mu_i -  \mu_{a}\right)\mathbb{E}[N_a(T)],
\end{equation}
where $N_a(t)$ is the number of pulling arm $a$ up to time $t$.
Thus, the attack goal means that the attacker wants the bandit algorithm to incur a linear regret by incurring only a sub-linear attack cost.

\begin{figure}[t]
  \centering
    \includegraphics[width=0.3\textwidth]{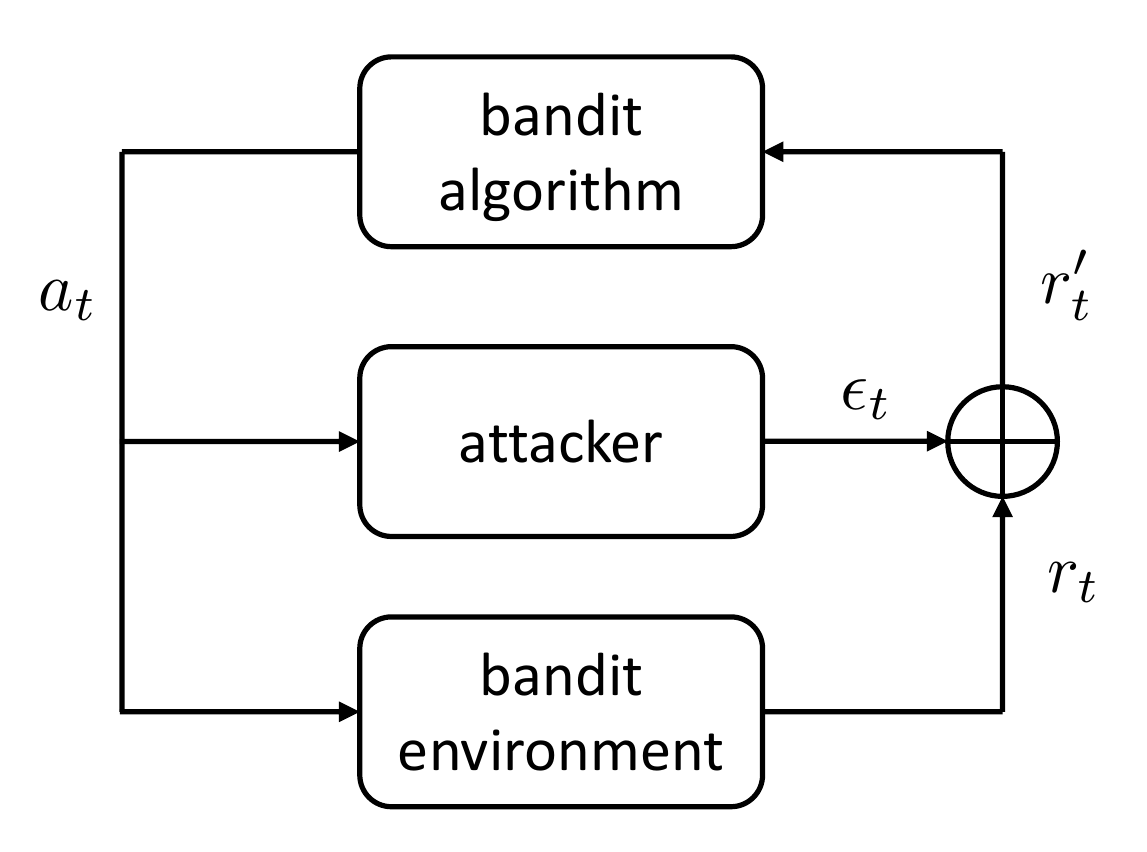}
     \caption{Online attack system model}
     \label{fig:onlineModel}
\end{figure}

\section{Offline Attacks}\label{sec:offline}
In this section, we introduce the offline attack framework to stochastic bandits. 
The updates of the bandit algorithm happen in mini-batches of size $T$\footnote{The batch size $T$ is a relatively large integer compared to $K$.}. Between these consecutive updates, the bandit algorithm follows a fixed algorithm obtained from the last update. This allows the attacker to poison the historical data before the update and force the algorithm to pull a target arm $a^*$ with high probability. 

Let $m_a$ be the number of times arm $a$ was pulled up to time $T$, i.e., $m_a \defeq N_a(T)$. For each arm $a\in\mathcal{A}$, let $\by_a\in\mathcal{R}^{m_a}$ be the corresponding reward vector returned from the bandit environment when arm $a$ was pulled. That is, $\by_a \defeq (r_t: a_t = a)^T$. Let $\bepsilon_a\in\mathcal{R}^{m_a}$ be the poisoning attack strategy of the attacker, i.e., $\bepsilon_a \defeq (\epsilon_t:a_t = a)^T$. The poisoned reward vector for arm $a$ after the attack becomes $\by_a+\bepsilon_a$. To avoid being detected, the attacker hopes to make the poisoning $\bepsilon_a$ as small as possible. Without loss of generality, we consider the $l^2$-norm attack cost in the offline attacks. We have that 
\begin{equation}\label{eqn:offCost}
C(T)^2 = \sum_{t=1}^T\epsilon^2_t = \sum_{a\in\mathcal{A}}||\bepsilon_a||_2^2.
\end{equation}
Thus, the attacker's offline attack problem can be formulated as the following optimization problem $P$,
\begin{align}
P: \min_{\bepsilon_a:a\in\mathcal{A}} &~~~~ \sum_{a\in\mathcal{A}}||\bepsilon_a||^2_2\\
s.t. &~~~~ \mathbb{P}\{a_{T+1} = a^*\} \geq 1-\delta,
\end{align}
for some error tolerance $\delta>0$.
Note that we define the attack goal as forcing the bandit algorithm to pull arm $a^*$ at the next round with high probability. This is because there are some randomized bandit algorithms, such as $\epsilon$-greedy and Thompson Sampling. It is not feasible to force the randomized algorithm to pull a target arm with probability~1. But it is possible to hijack the behavior of the randomized algorithm with high probability. 
\begin{proposition}\label{prop:offline}
Given some error tolerance $\delta>0$. If $\{\bepsilon_a^*\}_{a\in\mathcal{A}}$ is the optimal solution of problem $P$, then it is the optimal offline attack strategy for the attacker.
\end{proposition}

The proof of Proposition~\ref{prop:offline} follows from equation (\ref{eqn:offCost}) and the definition of offline attacks. 
Note that the constraint of problem $P$ depends on the bandit algorithm. Now, we assume that the attacker knows the bandit algorithm and we introduce algorithm-specific offline attack strategies derived from the optimization problem $P$ for three popular bandit algorithms, $\epsilon$-greedy, UCB and Thompson Sampling. For simplicity, we denote the post-attack empirical mean observed by the bandit algorithm at the end of round $t$ as 
\begin{equation}
\tilde{\mu}_a(t) := \frac{1}{N_a(t)}\sum_{\tau = 1}^t r'_\tau\mathds{1}\{a_\tau = a\}.
\end{equation}

\subsection{Offline attack on $\epsilon$-greedy}
Now, we derive the offline attack strategy for the $\epsilon$-greedy algorithm, which is a randomized algorithm with some decreasing rate function $\alpha_t$. Typically, $\alpha_t = \Theta(1/t)$. At each time $t$, the $\epsilon$-greedy algorithm pulls an arm 
\begin{eqnarray}\label{eqn:It}
a_t=
\begin{cases}
\text{draw uniformly over }\mathcal{A}, & \text{w.p. } \alpha_t\cr
\arg\max_{a\in\mathcal{A}}\tilde{\mu}_{a}(t-1), &\text{otherwise}
\end{cases}.
\end{eqnarray}

At time step $T+1$, the $\epsilon$-greedy algorithm uniformly samples an arm from the arm set $\mathcal{A}$ with probability $\alpha_{T+1}$, which cannot be controlled by the attacker. Then, we set the error tolerance as $\delta = \frac{K-1}{K}\alpha_{T+1}$ since the target arm $a^*$ may also be pulled in the uniform sampling. Thus, the attacker poisons the reward such that the target arm $a^*$ has the largest empirical mean. After the attack, the $\epsilon$-greedy algorithm pulls an arm $a_{T+1}$ at time $T+1$ such that $\mathbb{P}\{a_{T+1} = a^*\} = 1-\delta$. In order to make the target arm the unique arm with the largest empirical mean, we introduce a margin parameter $\xi>0$. So the attacker's optimization problem for attacking $\epsilon$-greedy algorithm is the following problem $P_1$,
\begin{align}
P_1: \min_{\bepsilon_a:a\in\mathcal{A}} &~~~~ \sum_{a\in\mathcal{A}}||\bepsilon_a||^2_2\\
s.t. &~~~~ \tilde{\mu}_{a^*}(T) \geq \tilde{\mu}_a(T)+\xi, ~~~~\forall a\not = a^*,
\end{align}
where $\tilde{\mu}_a(T) = (\by_a+\bepsilon_a)^T\mathds{1}/m_a$ and $\mathds{1}$ is the vector that each element is 1. 
The condition in the above optimization implies that the $\epsilon$-greedy algorithm will play the target arm $a^*$ at the time step $T+1$ with probability of at least $1-\frac{K-1}{K}\alpha_{T+1}$. The following result shows that there exists at least one optimal solution of problem $P_1$, i.e., one optimal offline attack for the $\epsilon$-greedy algorithm.

\begin{theorem}\label{thm:greedy}
Given any margin parameter $\xi>0$. For any reward instance $\{\by_a\}_{a\in\mathcal{A}}$, there exists at least one optimal solution of problem $P_1$, which is a quadratic program with linear constraints. Hence, there exists at least one optimal offline attack for the $\epsilon$-greedy algorithm.
\end{theorem}
%
%
The proof of Theorem~\ref{thm:greedy} is provided in Section~\ref{app:greedy} in the supplementary material.
Note that the error tolerance parameter, $\delta=\frac{K-1}{K}\alpha_{T+1}$, depends on the rate function $\alpha_t$ of the $\epsilon$-greedy algorithm, which is not controllable by the attacker. This counts the exploration introduced by the bandit algorithm's inner randomization, which can not be manipulated by the attacker. However, the attacker can wait for some large enough $T$ since the rate function $\alpha_t$ goes to zero finally. Moreover, the attacker's strategy (problem $P_1$) does not depend on the rate function, i.e., the attacker does not require the knowledge of the parameter of the $\epsilon$-greedy.
\subsection{Offline attack on UCB}
Now we derive the offline attack strategy for the classical UCB algorithm. For each time $t$, the UCB algorithm pulls the arm with the highest UCB index:
\begin{equation}
a_t = \arg\max_{a\in\mathcal{A}} u_a(t) := \tilde{\mu}_a(t-1) + 3\sigma\sqrt{\frac{\log t}{N_a(t-1)}}.
\end{equation}
The UCB algorithm pulls the target arm $a^*$ at time $T+1$ if and only if the UCB index of arm $a^*$ is the unique largest one. Thus, the attacker can manipulate the rewards to satisfy the condition. Given any margin parameter $\xi>0$, the attacker's optimization problem becomes
\begin{align}
P_2: \min_{\bepsilon_a:a\in\mathcal{A}} &~~~~ \sum_{a\in\mathcal{A}}||\bepsilon_a||^2_2\\\nonumber
s.t. &~~~~ u_{a^*}(T+1) \geq u_a(T+1) +\xi, ~~~~\forall a\not = a^*
\end{align}
The condition in the above optimization implies that the UCB algorithm will pull the target arm after the poisoning attack. The following result shows that there exists at least one optimal solution of problem $P_2$, i.e., one optimal offline attack for the UCB algorithm.

\begin{theorem}\label{thm:UCB}
Given any margin parameter $\xi>0$. For any reward instance $\{\by_a\}_{a\in\mathcal{A}}$, there exists at least one optimal solution of problem $P_2$, which is a quadratic program with linear constraints. Hence, there exists at least one optimal offline attack for the UCB algorithm.
\end{theorem}
The proof of Theorem~\ref{thm:UCB} is similar to the proof of Theorem~\ref{thm:greedy}. Note that the above attack strategy holds for any error tolerance $\delta$ since UCB algorithm is not randomized.

\subsection{Empirical mean based bandit algorithms}
One insight from our offline attacks on the $\epsilon$-greedy algorithm and the UCB algorithm is that the empirical mean based algorithms are vulnerable to attack. This is because the empirical mean related constraints are linear and non-empty. Then the optimization problem $P$ becomes a quadratic problem with non-empty linear constraints, which can be solved efficiently. This result holds for many other variants of UCB algorithms and Explore-Then-Commit algorithms~\cite{garivier2016explore} (which uniformly samples the arm set in the first exploration phase and commit to the arm with the largest empirical mean in the second commitment phase).

\subsection{Beyond empirical means: Thompson Sampling for Gaussian distributions}
We have shown that the empirical mean based algorithms are not secure against the offline attack. It would appear that Bayesian algorithms should be more robust to the offline attack as the constraint of problem $P$ becomes non-tractable. Unfortunately, we find that Thompson Sampling, a popular Bayesian algorithm, is also vulnerable to the data poisoning.

Now, we derive the attack strategy for Thompson Sampling for Gaussian distributions\footnote{Thompson Sampling needs distribution models and Gaussian distribution is popular and well-studied in the literature. The idea of problem relaxation here can be extended to other distributions.}. In other words, the noise $\eta_t$ is i.i.d. sampled from Gaussian distribution $\mathcal{N}(0,\sigma^2)$. Thompson Sampling for Gaussian distribution with the Jeffreys prior~\cite{korda2013thompson} works as the following. For each time $t$, the algorithm samples $\theta_a(t)$ from the posterior distribution $\mathcal{N}(\tilde{\mu}_a(t-1)/\sigma^2,\sigma^2/N_a(t-1))$ for each arm $a$ independently. Then the algorithm pulls the arm with the highest sample value, i.e., $a_t = \arg\max_{a\in\mathcal{A}}\theta_a(t)$.

Let $\phi(x) = \frac{1}{\sqrt{2\pi}}e^{-x^2/2}$ be the probability density function (pdf) of the standard Gaussian distribution and $\Phi(x)$ be the corresponding cumulative distribution function (cdf). For simplicity, we denote $\phi_a(x)$ as the pdf of the Gaussian distribution $\mathcal{N}(\tilde{\mu}_a(T)/\sigma^2,\sigma^2/N_a(T))$ for each arm $a$ and $\Phi_a(x)$ as the corresponding cdf. Since we are interested in the samples in time $T+1$, we omit time index in the following discussion. By the law of total expectation, we have that
\begin{align}
\mathbb{P}\{a_{T+1} &= a^*\} = \mathbb{P}\{\cap_{a\not = a^*} \theta_{a^*} > \theta_a\}\\\nonumber
& = \int_{-\infty}^{+\infty}\mathbb{P}\{\cap_{a\not = a^*} \theta_a < x | \theta_{a^*} = x\}\phi_{a^*}\left(x\right)dx\\\nonumber
& =\int_{-\infty}^{+\infty}\left(\Pi_{a\not = a^*}\mathbb{P}\{ \theta_a < x | \theta_{a^*} = x\}\right)\phi_{a^*}\left(x\right)dx\\
&=\int_{-\infty}^{+\infty}\left(\Pi_{a\not = a^*}\Phi_a(x)\right)\phi_{a^*}\left(x\right)dx.\label{eqn:directTS}
\end{align}
The Equation (\ref{eqn:directTS}) is complicated to compute and analyze, which makes the instantiation of the offline attack problem $P$ non-trivial. To address this challenge, we derive a sufficient constraint of the original constraint and make a relaxation of the original problem $P$. By the union bound, we have that
\begin{align}\nonumber
\mathbb{P}\{a_{T+1}\! \not=\! a^*\} &= \mathbb{P}\{\cup_{a\not=a^*}\theta_{a^*} \!<\! \theta_a \} \!\leq\!\sum_{a\not=a^*}\mathbb{P}\{\theta_{a^*}\!<\!\theta_a\}\\
&=\sum_{a\not=a^*}\Phi\left(\frac{\tilde{\mu}_a(T)-\tilde{\mu}_{a^*}(T)}{\sigma^3\sqrt{1/m_a+1/m_{a^*}}}\right)\label{eqn:unionTS}
\end{align}
Thus, we are ready to introduce the attacker's problem for Thompson Sampling for Gaussian distributions,
\begin{align}
P_3: \min_{\bepsilon_a:a\in\mathcal{A}} &~~~~ \sum_{a\in\mathcal{A}}||\bepsilon_a||^2_2\\
s.t. &~~~~ \sum_{a\not=a^*}\Phi\left(\frac{\tilde{\mu}_a(T)-\tilde{\mu}_{a^*}(T)}{\sigma^3\sqrt{1/m_a+1/m_{a^*}}}\right) \leq \delta \label{eqn:conTS}\\
&~~~~ \tilde{\mu}_a(T)-\tilde{\mu}_{a^*}(T)\leq0, ~~~~\forall a\not = a^*\label{eqn:meanTS}
\end{align}
The constraint of problem $P_3$ is a sufficient condition to the constraint of the problem $P$ by Equation (\ref{eqn:unionTS}). Note that the linear constraints (\ref{eqn:meanTS}) are redundant since $\Phi(0)=0.5$ and $\delta<\frac{K-1}{2}$ is usually satisfied. The following proposition shows that the constraint set of problem $P_3$ is convex.
\begin{proposition}\label{prop:convex}
The constraint set formed by Equations~(\ref{eqn:conTS}) and (\ref{eqn:meanTS}) is convex. 
\end{proposition}
The proof of Proposition~\ref{prop:convex} is provided in Section~\ref{app:convex} in the supplementary material. The following result shows that there exists at least one optimal solution of problem $P_3$ that is a feasible offline attack for the Thompson Sampling for Gaussian distributions.

\begin{theorem}\label{thm:TS}
Given any error tolerance $\delta>0$. For any reward instance $\{\by_a\}_{a\in\mathcal{A}}$, there exists at least one optimal solution of problem $P_3$, which is a quadratic program with convex constraints. Hence, there exists at least one feasible offline attack for the Thompson Sampling algorithm for Gaussian distributions.
\end{theorem}

By Proposition~\ref{prop:convex}, the proof of Theorem~\ref{thm:TS} is similar to the proof of Theorem~\ref{thm:greedy}. Note that the above attack strategy is not the optimal attack strategy formulated by $P$. However, it is easy to compute since problem $P_3$ is a quadratic program with convex constraints. Another relaxation of problem $P$ is presented in Section~\ref{app:P4} in the supplementary material.

\section{Online Attacks}\label{sec:online}
In this section, we study online attacks against stochastic bandits. 
In the online attack setting, the bandit algorithm updates its policy at each round. The attacker eavesdrops on the decision (i.e., $a_t$) of the bandit algorithm and poisons the reward by adding an arbitrary $\epsilon_t\in\mathcal{R}$. Hence the reward observed by the bandit algorithm is $r'_t = r_t+\epsilon_t.$ Without loss of generality, we consider the $l^1$-norm attack cost. That is, the cost of the attacker for round $t$ is $|\epsilon_t|$. Recall that $N_a(t)$ is the number of pulling arm $a$ up to time $t$.

Without the attacks, the bandit algorithm is a uniformly good policy such that it achieves $O(\log T)$ regret\footnote{The results in this section directly apply to the problem-independent regret bounds and high probability bounds.}, i.e., $\mathbb{E}[\sum_a(\max_i\mu_i-\mu_a)N_a(T)] = O(\log T)$ for any problem instance $\{\mu_a\}_{a\in\mathcal{A}}$. Moreover, the expected number of pulling the optimal arm (with the highest expected reward) is $T-o(T)$.

The goal of the attacker is to force the bandit algorithm to pull the sub-optimal target arm $a^*$ as much as possible and pays the least possible total cost. Formally, the attacker wants the bandit algorithm to receive linear expected regret, i.e., $\mathbb{E}[N_{a^*}(T)] = T - o(T)$, and pays the expected total cost $\mathbb{E}[\sum_t^T |\epsilon_t|] = O(\log T)$. In other words, the attacker wants to manipulate the rewards so that the bandit algorithm considers the target arm as the best arm in the long term.

\subsection{Oracle constant attacks}
The fact that the attacker does not know the expected rewards $\{\mu_a\}_{a\in\mathcal{A}}$ is challenging because otherwise the attacker can perform the attack trivially.  Suppose the attacker knows the expected rewards, then the attacker can choose the following oracle attack strategy,
\begin{equation}
\epsilon_t = - \mathds{1}\{a_t \not= a^*\}[\mu_{a_t}-\mu_{a^*}+\xi]^+,
\end{equation}
for some margin parameter $\xi>0$. Note that $[\cdot]^+$ indicates the function such that $[x]^+ = \max\{x,0\}$ and $\mathds{1}\{\cdot\}$ is the indicator function. By this attack, the bandit algorithm sees a poisoned bandit problem, where the target arm $a^*$ is the optimal arm and all the other arms are at least $\xi$ below the target arm. Thus, the bandit algorithm pulls the target arm with $\mathbb{E}[N_{a^*}(T)] = T - o(T)$ and pays the total cost $\mathbb{E}[\sum_t^T |\epsilon_t|] = O(\log T)$ since $\epsilon_t$ are bounded. This result has been shown in~\cite{jun2018adversarial} as the following.

\begin{proposition}\label{prop:oracle}
(Proposition 1 in~\cite{jun2018adversarial})
Assume that the bandit algorithm achieves an $O(\log T)$ regret bound. Then the oracle attack with $\xi>0$ succeeds, i.e., $\mathbb{E}[N_{a^*}(T)] = T-o(T)$. Furthermore, the expected attack cost is $O(\sum_{i\not= a^*}[\mu_{i}-\mu_{a^*}+\xi]^+\log T)$.
\end{proposition}

The insight obtained from Proposition~\ref{prop:oracle} is that the attacker does not need to attack in round $t$ if $a_t=a^*$. This helps us to design attack strategies that are similar to the oracle attack. When the ground truth $\{\mu_a\}_{a\in\mathcal{A}}$ is not known to the attacker, the attacker may guess some upper bound $\{C_a\}_{a\not = a^*}$ on $\{[\mu_{a}-\mu_{a^*}]^+\}_{a\not = a^*}$ and perform the following oracle constant attack,
\begin{equation}
\epsilon_t = - \mathds{1}\{a_t \not= a^*\}C_{a_t}.
\end{equation}
The following result shows the sufficient and necessary conditions for the oracle constant attack to be successful.

\begin{proposition}\label{prop:constant}
Assume that the bandit algorithm achieves an $O(\log T)$ regret bound. Then the constant attack with $\{C_a\}_{a\not=a^*}$ succeeds if and only if $C_a > [\mu_{a}-\mu_{a^*}]^+$, $\forall a \not= a^*$. If the attack succeeds, then the expected attack cost is $O(\sum_{a\not= a^*}C_a\log T)$.
\end{proposition}
%
The proof of Proposition~\ref{prop:constant} is provided in Section~\ref{app:constant} in the supplementary material.
Proposition~\ref{prop:constant} shows that the attacker has to know the unknown bounds on $\{[\mu_{a}-\mu_{a^*}]^+\}_{a\not = a^*}$ to guarantee a successful constant attack. Moreover, the oracle constant attack is non-adaptive to the problem instance since some upper bound $C_a$ can be much larger than the quantity $[\mu_{a}-\mu_{a^*}]^+$ so that the attacker is paying unnecessary attack cost compared to the oracle attack. To address this challenge, we propose an adaptive constant attack that can learn the bandit environment in an online fashion and perform the attack adaptively. 

\subsection{Adaptive attack by constant estimation}
Now, we are ready to propose the adaptive attack strategy. The idea is that the attacker can update upper bounds on the unknown quantity $\{[\mu_{a}-\mu_{a^*}]^+\}_{a\not = a^*}$ based on the empirical mean observed by the attacker. Then the attacker performs the constant attack based on the estimated upper bounds. Note that the attacker observes the tuple $(a_t,r_t)$ at each time $t$ and is able to obtain an unbiased empirical mean. 
Let $\hat{\mu}_a(t)$ be the pre-attack empirical mean observed by the attacker at time $t$, that is
\begin{equation}
\hat{\mu}_a(t) := \frac{1}{N_a(t)}\sum_{\tau = 1}^t r_\tau\mathds{1}\{a_\tau = a\}.
\end{equation}
Given any $\delta\in(0,1)$, we define a function $\beta(n)$ as 
\begin{equation}
\beta(n) = \sqrt{\frac{2\sigma^2}{n}\log\frac{\pi^2Kn^2}{3\delta}}.
\end{equation}
Then, the following result shows that $\{\beta(N_a(t))\}_{a\in\mathcal{A}}$ is a high-probability bound on the pre-attack empirical mean of all arms on all rounds. Define the event
\begin{equation}
E = \{\forall a\in\mathcal{A}, \forall t: |\hat{\mu}_a(t) - \mu_a| < \beta(N_a(t))\}.
\end{equation}
\begin{lemma}\label{lem:concen}
(Lemma 1 in~\cite{jun2018adversarial}) For $\delta\in(0,1)$, $\mathbb{P}(E)>1-\delta$.
\end{lemma}
%
The proof of Lemma~\ref{lem:concen} follows from Hoeffding inequality and union bound. Lemma~\ref{lem:concen} shows that with probability at least $1-\delta$, we have that the expected reward $\mu_a$ is within the interval $(\hat{\mu}_a(t)-\beta(N_a(t)),\hat{\mu}_a(t)+\beta(N_a(t)))$ centered at the pre-attack empirical mean $\hat{\mu}_a(t)$. Thus, we obtain an upper bound on the quantity $\mu_a-\mu_{a^*}$ and propose the following Adaptive attack by Constant Estimation (ACE), 
\begin{equation}\nonumber
\epsilon_t \!=\! -\! \mathds{1}\{a_t\! \not=\! a^*\}[\hat{\mu}_{a_t}\!(t)-\hat{\mu}_{a^*}\!(t)+\beta(N_{a_t}\!(t))+\beta(N_{a^*}\!(t))]^+.
\end{equation}
The following shows that the ACE is successful for any bandit algorithm \emph{without the knowledge of the bandit algorithm}.

\begin{theorem}\label{thm:adaptive}
Given any $\delta\in(0,0.5)$, assume that the bandit algorithm achieves an $O(\log T)$ regret bound with probability at least $1-\delta$. With probability at least $1-2\delta$, the ACE attacker forces the bandit algorithm to play the target arm $a^*$ in $N_{a^*}(T)$ times, such that $N_{a^*}(T) = T-o(T)$, using an instantaneous attack cost,
\begin{equation}\nonumber
|\epsilon_t|\! \leq\!  \mathds{1}\{a_t\! \not=\! a^*\}\!\left([\mu_{a_t}\!\!-\!\mu_{a^*}]^+ \!\!+\! 2\beta(N_{a_t}\!(t)) \!+\! 2\beta(N_{a^*}\!(t))\right) .
\end{equation}
In particular, with probability at least $1-2\delta$, the ACE attack is successful and the accumulated attack cost is
\begin{equation}\nonumber
\sum_{t=1}^T|\epsilon_t| \leq O\left(\sum_{a\not= a^*}\left([\mu_a-\mu_{a^*}]^+ + 4\beta(1)\right) \log T\right).
\end{equation}
\end{theorem}

The proof of Theorem~\ref{thm:adaptive} is provided in Section~\ref{app:adaptive} in the supplementary material.
Theorem~\ref{thm:adaptive} shows that the ACE is universally successful for any bandit algorithm, without knowing any prior information on $\{\mu_a\}_{a\in\mathcal{A}}$.
Besides, the ACE incurs an high-probability accumulated attack cost as small as that of the oracle attack\footnote{High-probability bounds can be adapted from Proposition~\ref{prop:oracle}. } with only an additional bounded additive term, $O(4\beta(1)K\log T)$. That is, the ACE is close to the hindsight-oracle attack strategy. 
Moreover, the ACE even requires no knowledge of the bandit algorithm. This is an advantage over the algorithm-dependent online attack strategies designed in~\cite{jun2018adversarial} since the attacker may not know which bandit algorithm the learner is in practice. As far as we know, this is the first negative result showing that there is no robust bandit algorithm that can be immune to the adaptive online attack. This exposes a significant security threat to the bandit learning systems.

\section{Numerical Results}\label{sec:simul}

\begin{figure*}[t]
\centering
\begin{subfigure}[b]{0.3\textwidth}
    \includegraphics[width=\textwidth]{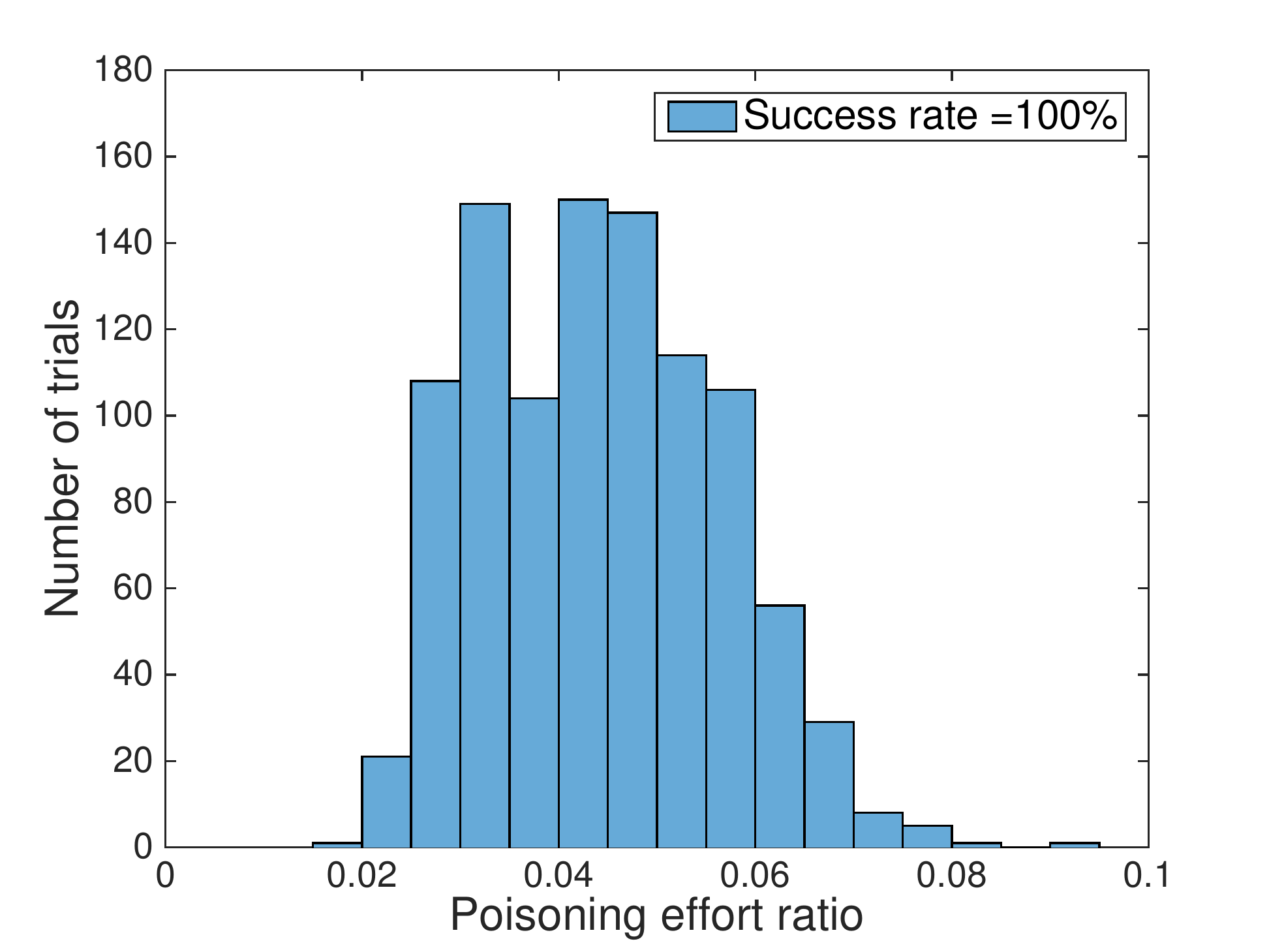}
     \caption{Attack on $\epsilon$-greedy algorithm}
     \label{fig:offlineGreedy}
\end{subfigure}
\begin{subfigure}[b]{0.3\textwidth}
    \includegraphics[width=\textwidth]{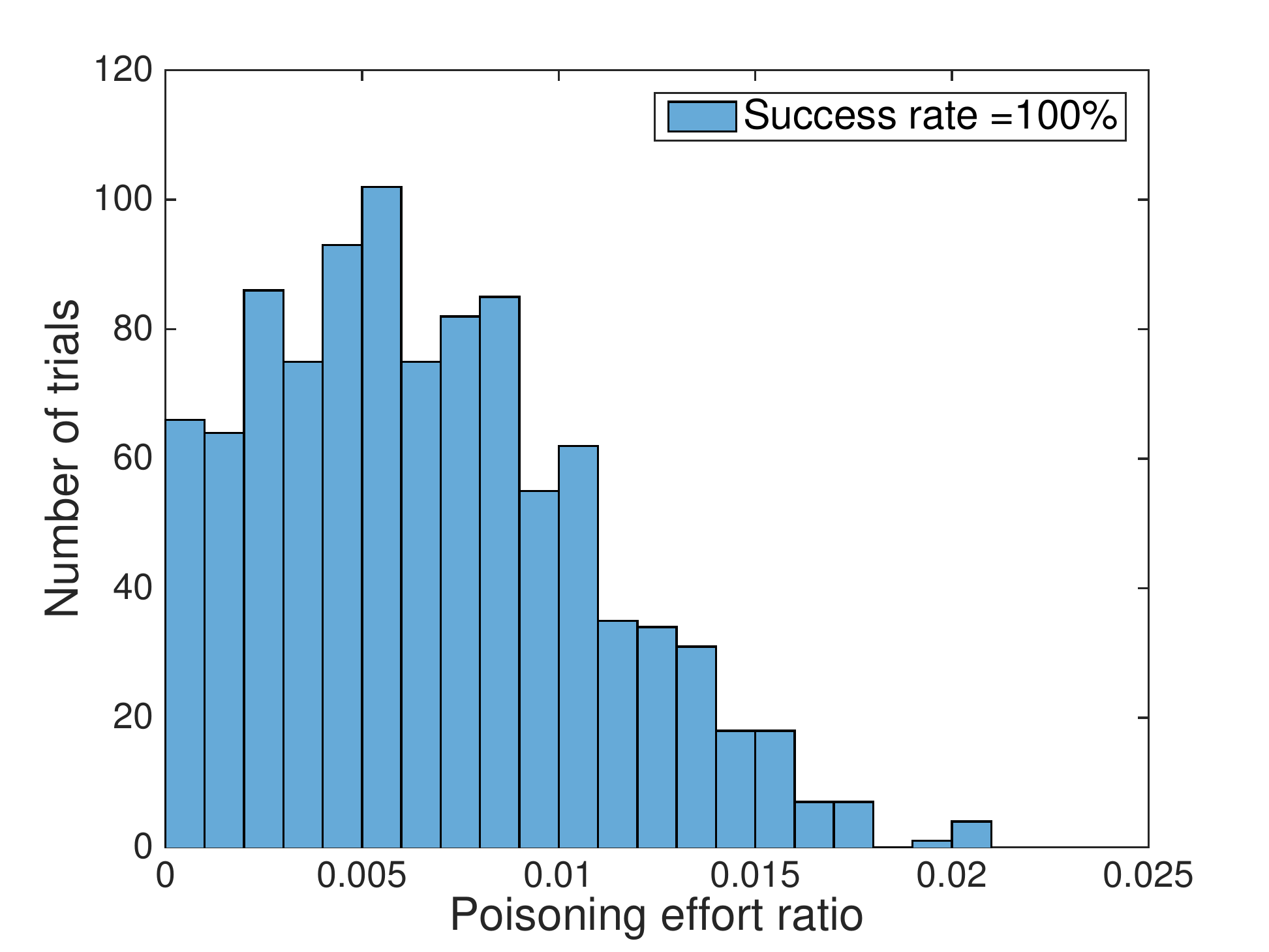}
     \caption{Attack on UCB algorithm}
     \label{fig:offlineUCB}
\end{subfigure}
\begin{subfigure}[b]{0.3\textwidth}
    \includegraphics[width=\textwidth]{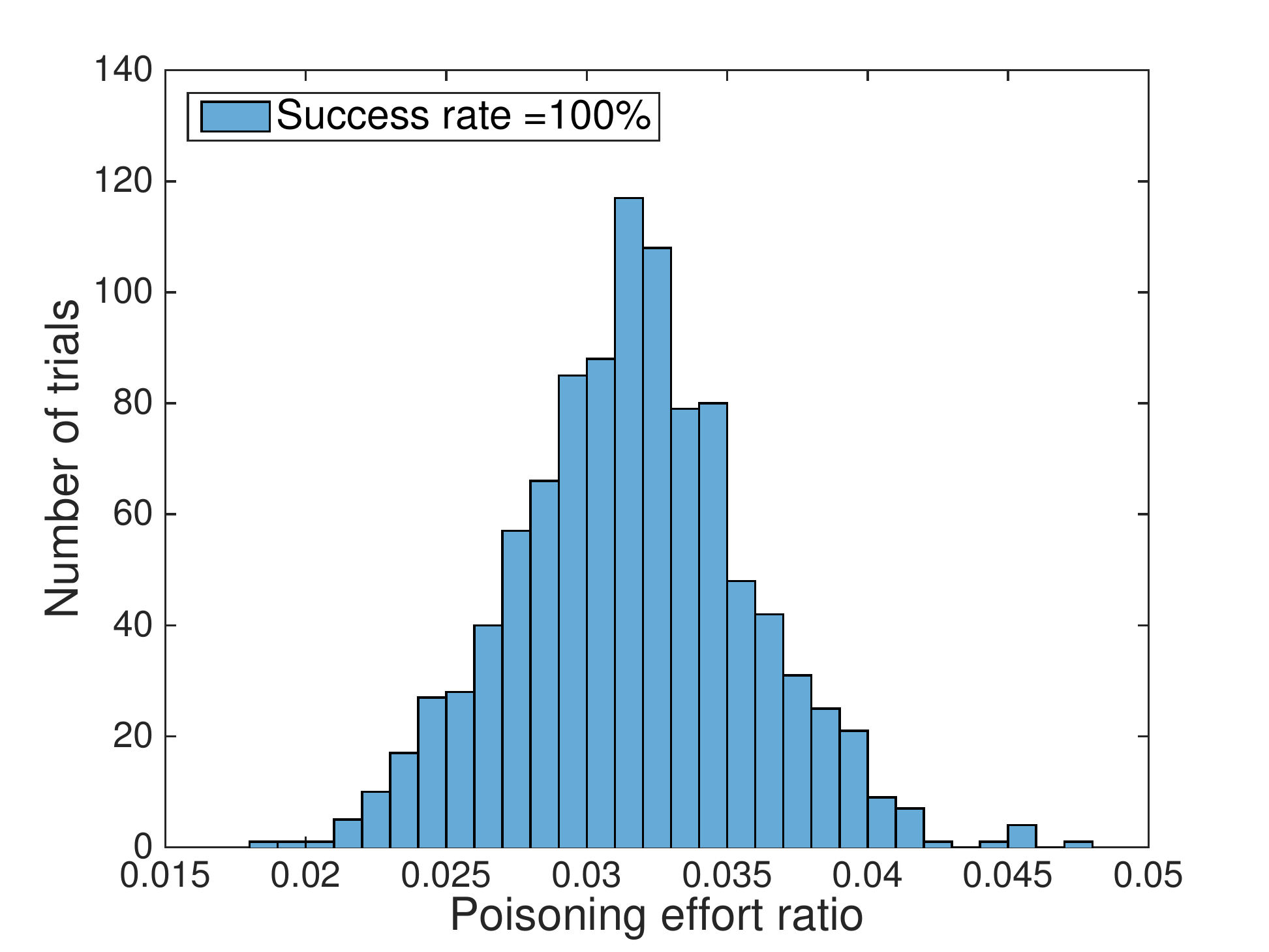}
     \caption{Attack on Thompson Sampling}
     \label{fig:offlineTS}
\end{subfigure}
\caption{Histograms of poisoning effort ratio in the offline attacks.}
\label{fig:offline}
\end{figure*}

In this section, we run simulations on attacking $\epsilon$-greedy, UCB and Thompson Sampling algorithms to illustrate our theoretical results. All the simulations are run in MATLAB and the codes can be found in the supplemental materials. Note that we implement $\epsilon$-greedy with $\alpha_t= \frac{1}{t}$.
\subsection{Offline attacks}
To study the effectiveness of the offline attacks, we consider the following experiment. The bandit has $K=5$ arms and the reward noise is a Gaussian distribution $\mathcal{N}(0,\sigma^2)$ with $\sigma = 0.1$. The attacker's target arm is arm $K$ such that $\mu_K= 0$, while the expected rewards of other arms are uniformly distributed in the unit interval $[0,1]$. We set $T=1000$ and the error tolerance to $\delta = 0.05$.

In each attack trial, we first generate the instance of the bandit by drawing the expected rewards from the uniform distribution on $[0,1]$. Then we run the bandit algorithm for $T$ rounds and collect all the historical data. Without any attack, the bandit algorithm would have converged to some optimal arm in the trial, which is not the target arm as the target arm is the one with the least payoff. Then we set the margin parameter as $\xi=0.001$ and run the corresponding offline attacks. The attack strategy is the solution of the optimization problem $P_i$. Since all the problems are quadratic program with linear (convex) constraints, they can be solved by standard optimization tools. 

We run 1000 attack trials. Note that the attack cost depends on the instance of the bandit. To evaluate the attack cost, we use the poisoning effort ratio~\cite{ma2018data}:
\begin{equation}
\frac{||\bepsilon||_2}{||\by||_2}=\sqrt{\frac{\sum_{a\in\mathcal{A}}||\bepsilon_a||^2_2}{\sum_{a\in\mathcal{A}}||\by_a||^2_2}}.
\end{equation}
The poisoning effort ratio measures the ratio of the total cost over the norm of the original rewards. To evaluate the attack effectiveness, we check whether the poisoned data satisfy the constraint of the optimization problem $P$. If so, the bandit algorithm will play the target arm with probability at least $1-\delta$. 

Figure~\ref{fig:offline} shows the results of the attack against the $\epsilon$-greedy, UCB and Thompson Sampling. First, the attack strategies are effective as all the attacks are successful. Second, the attack costs are small as shown in the histograms. The ratios of attacking $\epsilon$-greedy, UCB and Thompson Sampling are less than $10\%$, $2\%$ and $5\%$.

\subsection{Online attacks}
\begin{figure*}[t]
\centering
\begin{subfigure}[b]{0.3\textwidth}
    \includegraphics[width=\textwidth]{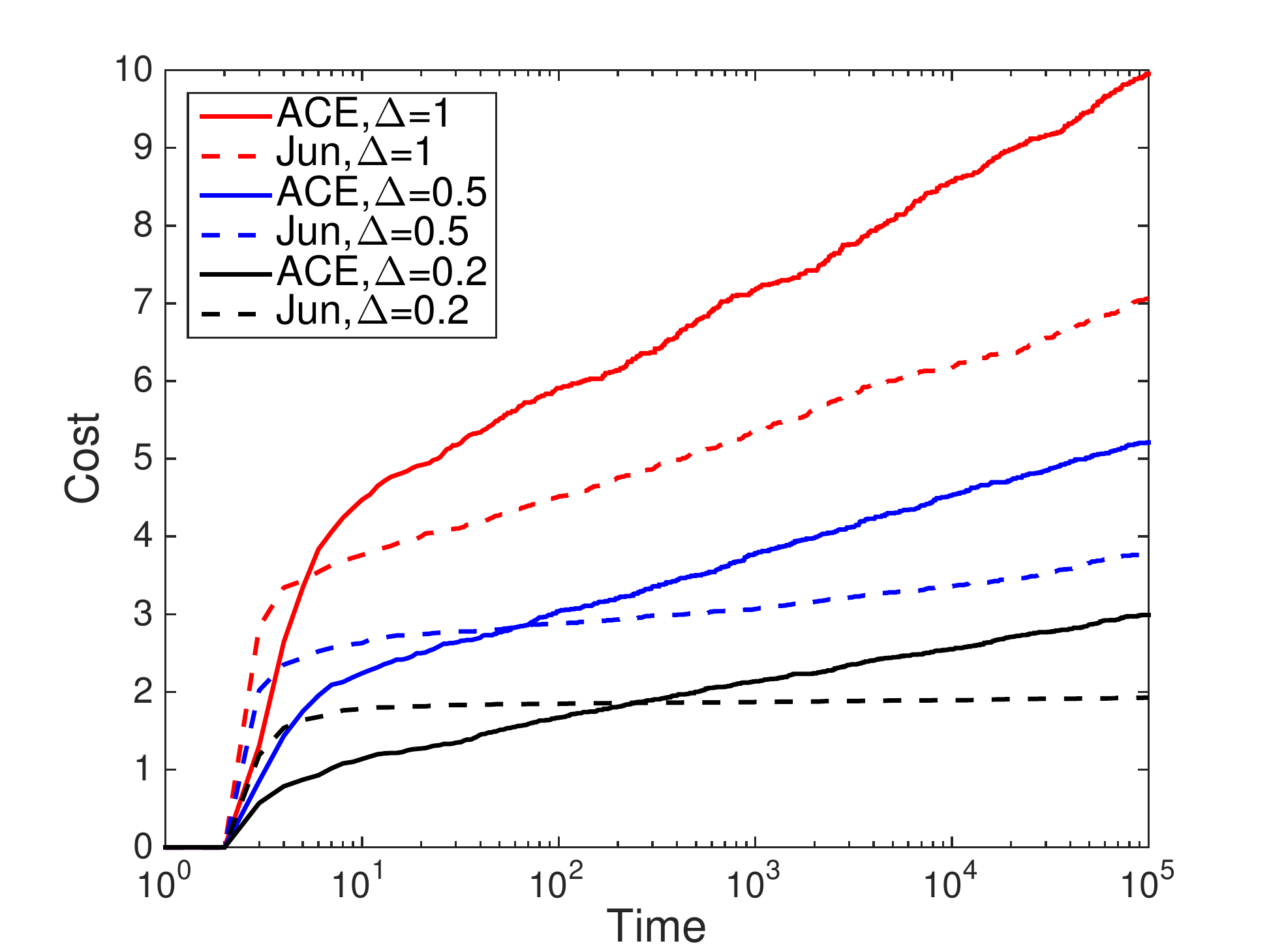}
\end{subfigure}
\begin{subfigure}[b]{0.3\textwidth}
    \includegraphics[width=\textwidth]{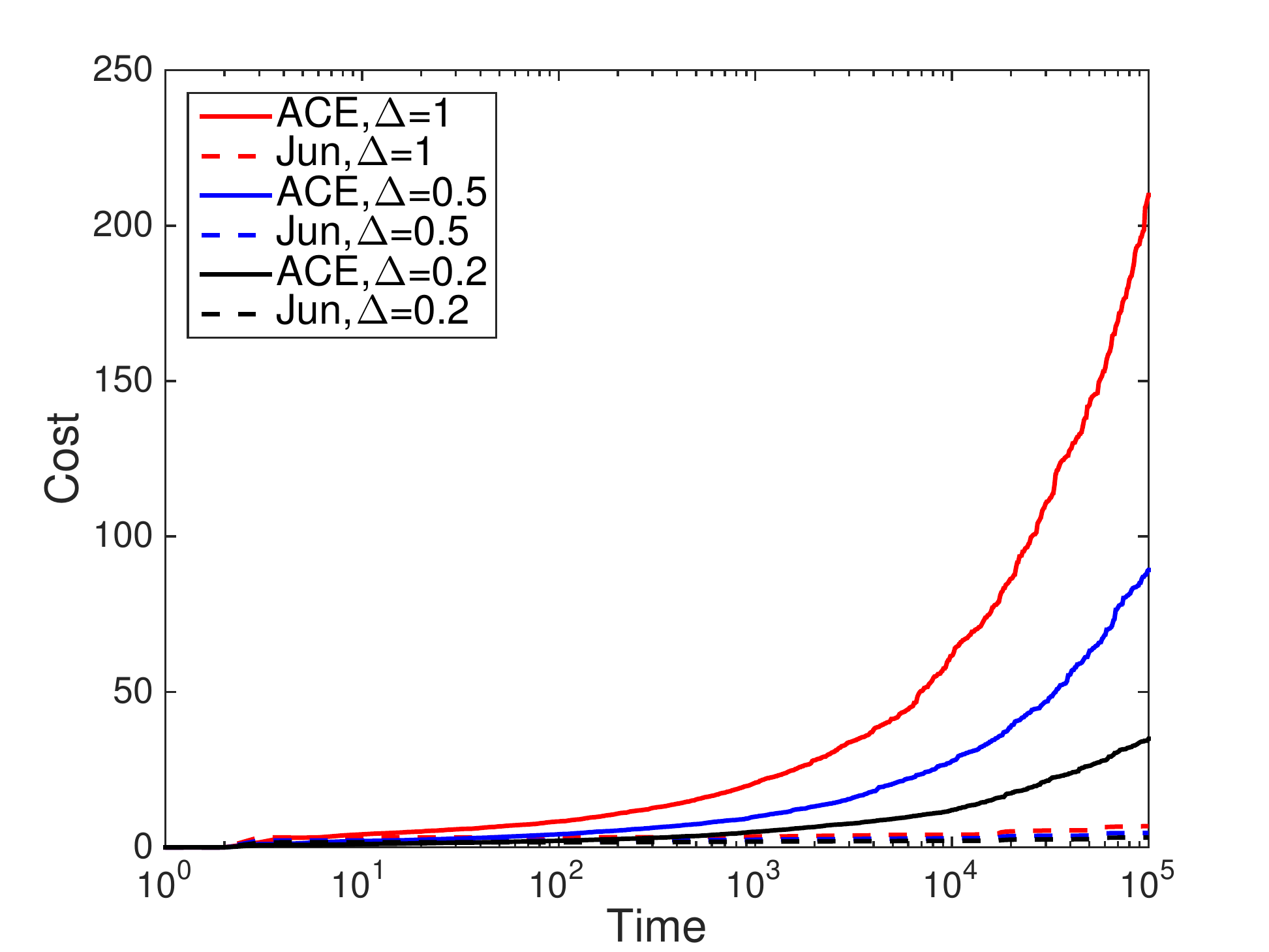}
\end{subfigure}
\begin{subfigure}[b]{0.3\textwidth}
    \includegraphics[width=\textwidth]{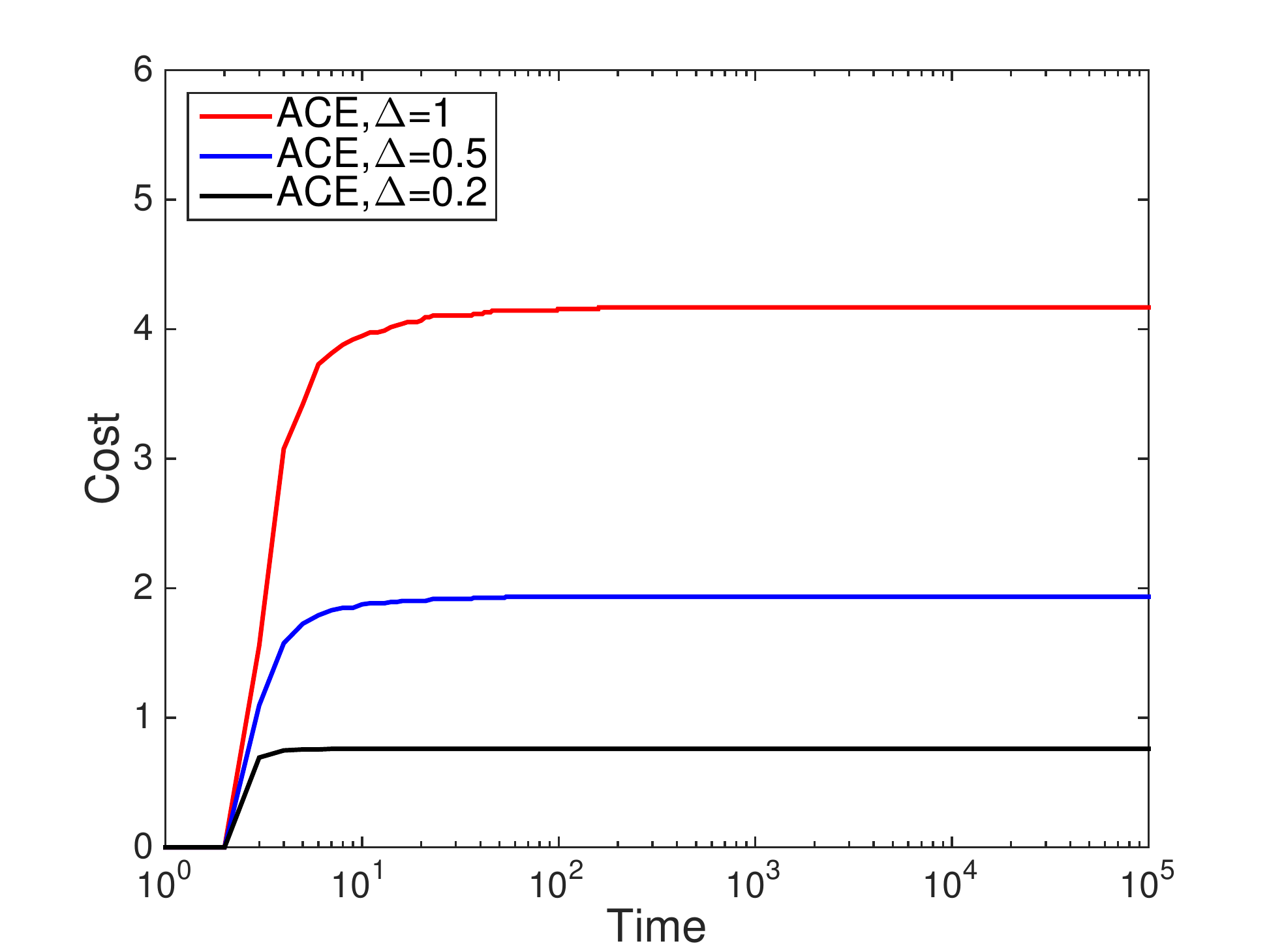}
\end{subfigure}

\begin{subfigure}[b]{0.3\textwidth}
    \includegraphics[width=\textwidth]{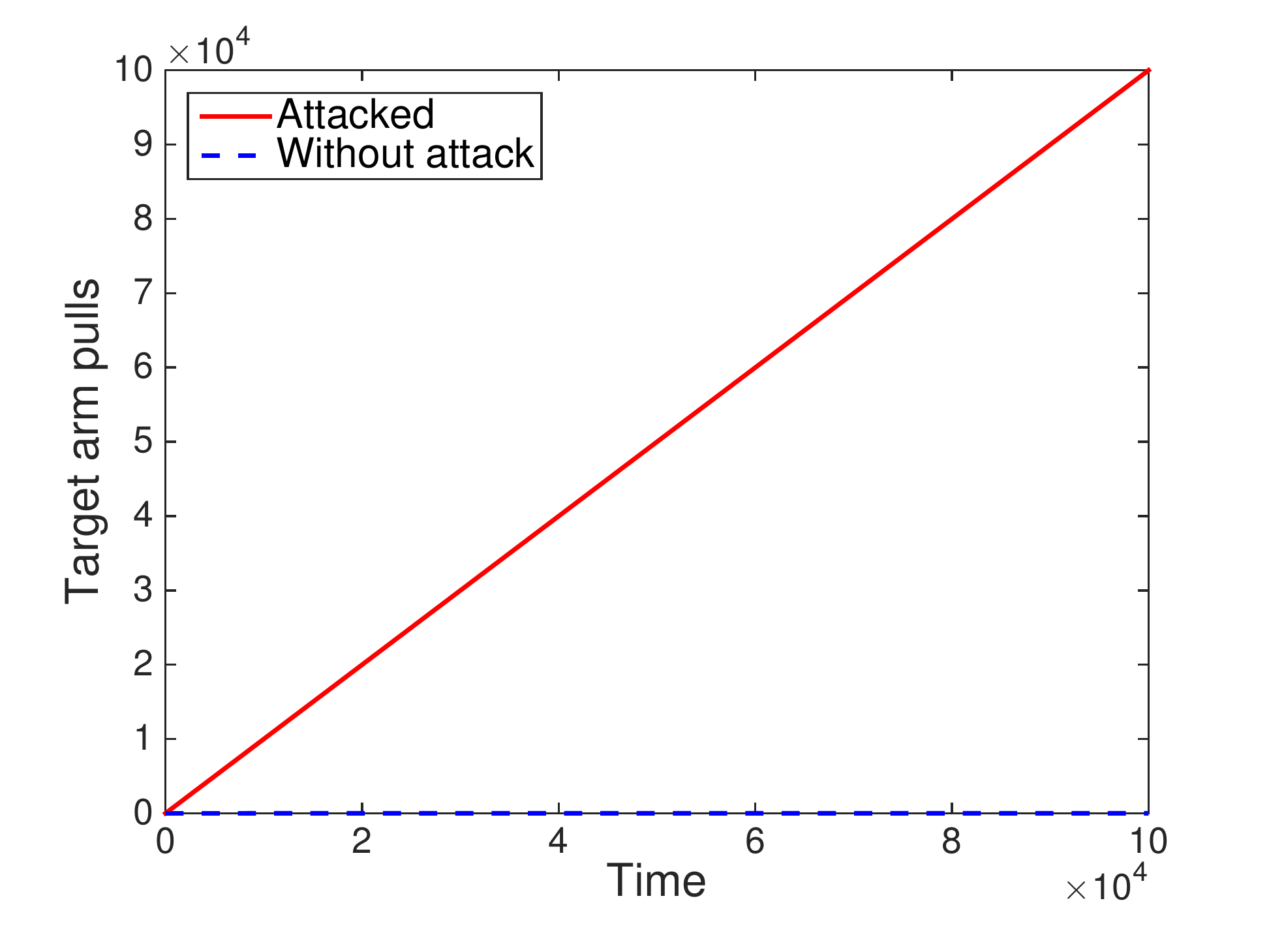}
     \caption{Attack on $\epsilon$-greedy algorithm}
     \label{fig:onlineGreedy}
\end{subfigure}
\begin{subfigure}[b]{0.3\textwidth}
    \includegraphics[width=\textwidth]{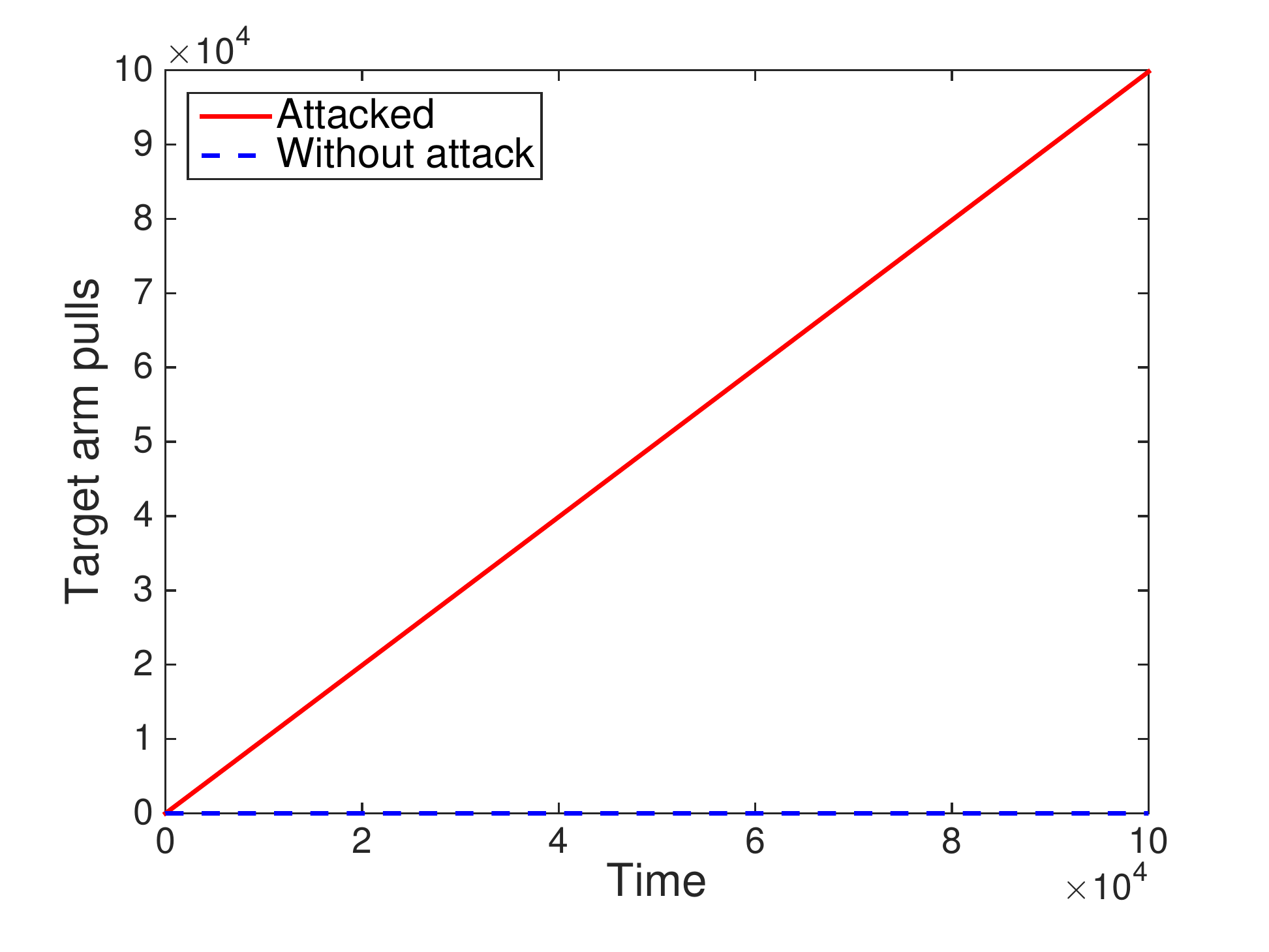}
     \caption{Attack on UCB algorithm}
     \label{fig:onlineUCB}
\end{subfigure}
\begin{subfigure}[b]{0.3\textwidth}
    \includegraphics[width=\textwidth]{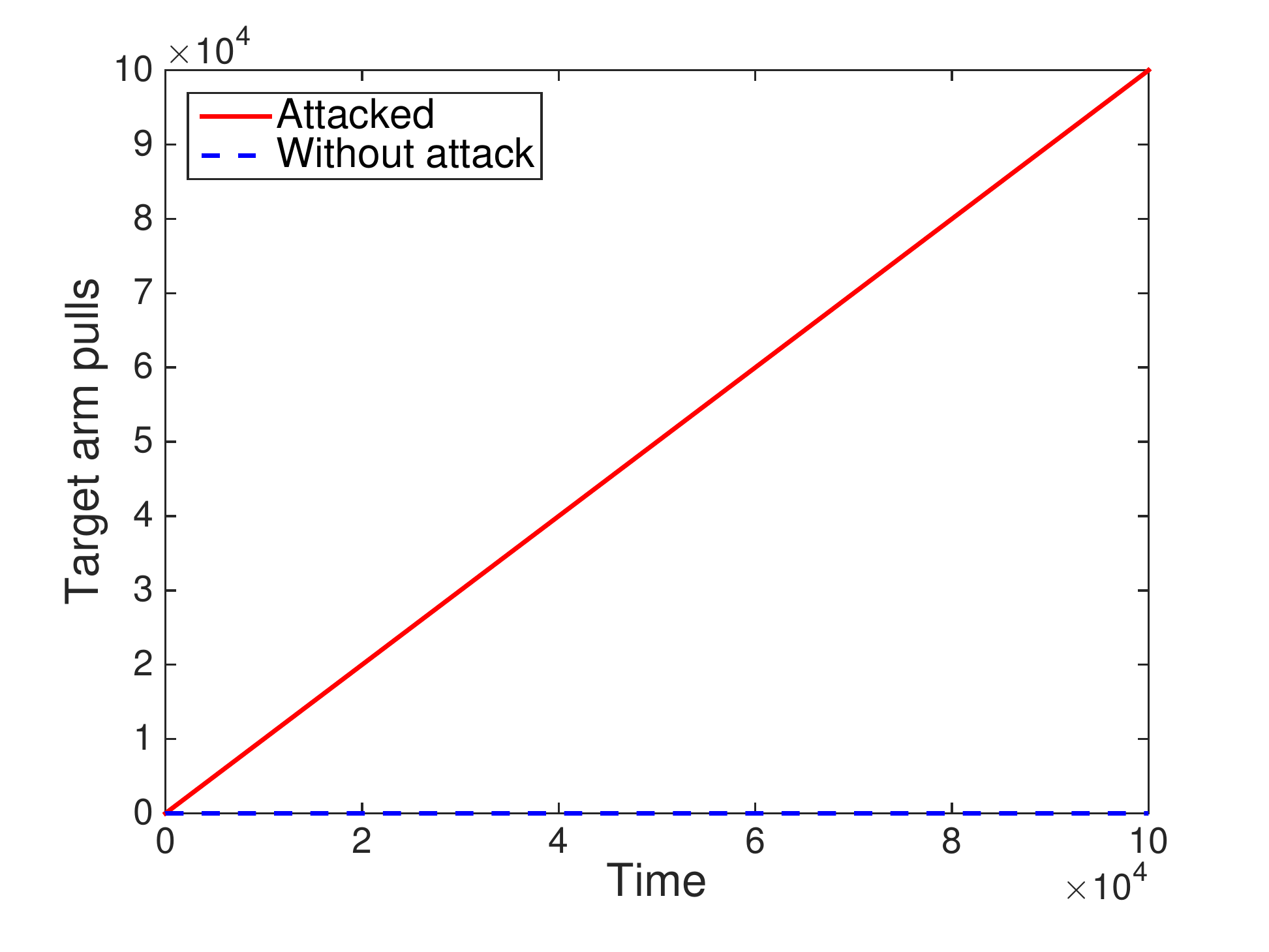}
     \caption{Attack on Thompson Sampling}
     \label{fig:onlineTS}
\end{subfigure}
\caption{Attack cost and target arm pulls in the online attacks.}
\label{fig:online}
\end{figure*}

We compare our adaptive attack strategy with two attack strategies proposed by~\cite{jun2018adversarial}. Note that these two attacks are against $\epsilon$-greedy and UCB algorithm and require the knowledge of the algorithm. In contrast, we highlight that our attack strategy ACE is an universal attack strategy against any bandit algorithm. 

We consider the following experiment. The bandit has two arms. The expected rewards of arms 1 and 2 are $\mu_1=\Delta$ and $\mu_2=0$ with $\Delta>0$. The attacker's target is arm 2. The noise of the rewards is i.i.d. sampled from the Gaussian distribution $\mathcal{N}(0,\sigma^2)$ with $\sigma= 0.1$. We set the error tolerance to $\delta = 0.05$ and time horizon to $T = 10^5$ rounds. For the implementations of the attack strategies proposed by~\cite{jun2018adversarial}, we choose the tuning parameter $\Delta_0=\sigma$, which is suggested by~\cite{jun2018adversarial} when $T$ is not known to the attacker. We run 100 attack trials with different $\Delta$ values. 

Figure~\ref{fig:online} shows the average attack cost and number of target arm pulls in the online attacks. Note that the target arm pulls are the cases when $\Delta = 1$. First, we compare the number of target arm pulls with ACE attack and without. ACE attack dramatically forces the bandit algorithm to pull the target arm linearly in time. Second, the attack costs are relatively small compared to the loss of the bandit algorithm, which is linear in time. Generally, the attack costs of ACE attack are bounded by $O(\log T)$ and increase as the reward gap $\Delta$ becomes larger. These verify the result of Theorem~\ref{thm:adaptive}. On the other hand, the attack costs on Thompson Sampling and $\epsilon$-greedy are relatively smaller than that of UCB. This is because Thompson Sampling and $\epsilon$-greedy converges to the ``optimal" arm very fast while the exploration for ``non-optimal" arm may still increase over time. Finally, compared to the algorithm-specific attacks proposed by~\cite{jun2018adversarial}, the attack cost of ACE on $\epsilon$-greedy is slightly worse while the attack cost of ACE on UCB is much larger than Jun's attack. In the case of attacking UCB algorithm, our universal attack strategy takes more cost than the algorithm-specific attack, but without the need to know the algorithm.

\section{Conclusion}\label{sec:conclusion}
In this work, we study the open problem of data poisoning attacks on bandit algorithms. We propose an offline attack framework for the stochastic bandits and propose three algorithm-specific offline attack strategies against $\epsilon$-greedy, UCB and Thompson Sampling. Then, we study an online attack against the bandit algorithms and propose the adaptive attack strategy that can hijack the behavior of any bandit algorithm without requiring the knowledge of the bandit algorithm. Our theoretical results and simulations show that the bandit algorithms are vulnerable to the poisoning attacks in both online and offline setting.

\section*{Acknowledgements}
This work has been supported in part by a grant from DTRA (HDTRA1-14-1-0058), a grant from Office of Naval Research (N00014-17-1-2417), grants from National Science foundation (CNS-1446582, CNS-1518829, CNS-1409336) and a grant from Army Research Office (WN11NF-15-1-0277). This work has also been supported by Institute for Information \& communications Technology Promotion (IITP) grant funded by the Korea government (MSIT), (2017-0-00692, Transport-aware Streaming Technique Enabling Ultra Low-Latency AR/VR Services). 
\bibliography{refs}

\begin{thebibliography}{19}
\providecommand{\natexlab}[1]{#1}
\providecommand{\url}[1]{\texttt{#1}}
\expandafter\ifx\csname urlstyle\endcsname\relax
  \providecommand{\doi}[1]{doi: #1}\else
  \providecommand{\doi}{doi: \begingroup \urlstyle{rm}\Url}\fi

\bibitem[Alfeld et~al.(2016)Alfeld, Zhu, and Barford]{alfeld2016data}
Alfeld, S., Zhu, X., and Barford, P.
\newblock Data poisoning attacks against autoregressive models.
\newblock In \emph{AAAI}, pp.\  1452--1458, 2016.

\bibitem[Biggio et~al.(2012)Biggio, Nelson, and Laskov]{biggio2012poisoning}
Biggio, B., Nelson, B., and Laskov, P.
\newblock Poisoning attacks against support vector machines.
\newblock In \emph{Proceedings of the 29th International Conference on Machine
  Learning}, pp.\  1467--1474, 2012.

\bibitem[Chapelle et~al.(2015)Chapelle, Manavoglu, and
  Rosales]{chapelle2015simple}
Chapelle, O., Manavoglu, E., and Rosales, R.
\newblock Simple and scalable response prediction for display advertising.
\newblock \emph{ACM Transactions on Intelligent Systems and Technology (TIST)},
  5\penalty0 (4):\penalty0 61, 2015.

\bibitem[Garivier et~al.(2016)Garivier, Lattimore, and
  Kaufmann]{garivier2016explore}
Garivier, A., Lattimore, T., and Kaufmann, E.
\newblock On explore-then-commit strategies.
\newblock In \emph{Advances in Neural Information Processing Systems}, pp.\
  784--792, 2016.

\bibitem[Goodfellow et~al.(2015)Goodfellow, Shlens, and Szegedy]{43405}
Goodfellow, I., Shlens, J., and Szegedy, C.
\newblock Explaining and harnessing adversarial examples.
\newblock In \emph{International Conference on Learning Representations}, 2015.
\newblock URL \url{http://arxiv.org/abs/1412.6572}.

\bibitem[Gupta et~al.(2019)Gupta, Koren, and Talwar]{gupta2019better}
Gupta, A., Koren, T., and Talwar, K.
\newblock Better algorithms for stochastic bandits with adversarial
  corruptions.
\newblock \emph{arXiv preprint arXiv:1902.08647}, 2019.

\bibitem[Huang et~al.(2017)Huang, Papernot, Goodfellow, Duan, and
  Abbeel]{huang2017adversarial}
Huang, S., Papernot, N., Goodfellow, I., Duan, Y., and Abbeel, P.
\newblock Adversarial attacks on neural network policies.
\newblock \emph{arXiv preprint arXiv:1702.02284}, 2017.

\bibitem[Jun et~al.(2018)Jun, Li, Ma, and Zhu]{jun2018adversarial}
Jun, K.-S., Li, L., Ma, Y., and Zhu, X.
\newblock Adversarial attacks on stochastic bandits.
\newblock \emph{arXiv preprint arXiv:1810.12188}, 2018.

\bibitem[Korda et~al.(2013)Korda, Kaufmann, and Munos]{korda2013thompson}
Korda, N., Kaufmann, E., and Munos, R.
\newblock Thompson sampling for 1-dimensional exponential family bandits.
\newblock In \emph{Advances in Neural Information Processing Systems}, pp.\
  1448--1456, 2013.

\bibitem[Li et~al.(2016)Li, Wang, Singh, and Vorobeychik]{li2016data}
Li, B., Wang, Y., Singh, A., and Vorobeychik, Y.
\newblock Data poisoning attacks on factorization-based collaborative
  filtering.
\newblock In \emph{Advances in neural information processing systems}, pp.\
  1885--1893, 2016.

\bibitem[Li et~al.(2010)Li, Chu, Langford, and Schapire]{li2010contextual}
Li, L., Chu, W., Langford, J., and Schapire, R.~E.
\newblock A contextual-bandit approach to personalized news article
  recommendation.
\newblock In \emph{Proceedings of the 19th international conference on World
  wide web}, pp.\  661--670. ACM, 2010.

\bibitem[Lin et~al.(2017)Lin, Hong, Liao, Shih, Liu, and Sun]{lin2017tactics}
Lin, Y.-C., Hong, Z.-W., Liao, Y.-H., Shih, M.-L., Liu, M.-Y., and Sun, M.
\newblock Tactics of adversarial attack on deep reinforcement learning agents.
\newblock \emph{arXiv preprint arXiv:1703.06748}, 2017.

\bibitem[Liu et~al.(2018)Liu, Wang, Buccapatnam, and Shroff]{liu2018ucboost}
Liu, F., Wang, S., Buccapatnam, S., and Shroff, N.
\newblock Ucboost: a boosting approach to tame complexity and optimality for
  stochastic bandits.
\newblock \emph{arXiv preprint arXiv:1804.05929}, 2018.

\bibitem[Lykouris et~al.(2018)Lykouris, Mirrokni, and
  Paes~Leme]{lykouris2018stochastic}
Lykouris, T., Mirrokni, V., and Paes~Leme, R.
\newblock Stochastic bandits robust to adversarial corruptions.
\newblock In \emph{Proceedings of the 50th Annual ACM SIGACT Symposium on
  Theory of Computing}, pp.\  114--122. ACM, 2018.

\bibitem[Ma et~al.(2018)Ma, Jun, Li, and Zhu]{ma2018data}
Ma, Y., Jun, K.-S., Li, L., and Zhu, X.
\newblock Data poisoning attacks in contextual bandits.
\newblock \emph{arXiv preprint arXiv:1808.05760}, 2018.

\bibitem[Mei \& Zhu(2015)Mei and Zhu]{mei2015using}
Mei, S. and Zhu, X.
\newblock Using machine teaching to identify optimal training-set attacks on
  machine learners.
\newblock In \emph{AAAI}, pp.\  2871--2877, 2015.

\bibitem[Thompson(1933)]{thompson1933likelihood}
Thompson, W.~R.
\newblock On the likelihood that one unknown probability exceeds another in
  view of the evidence of two samples.
\newblock \emph{Biometrika}, 25\penalty0 (3/4):\penalty0 285--294, 1933.

\bibitem[Wang \& Chaudhuri(2018)Wang and Chaudhuri]{wang2018data}
Wang, Y. and Chaudhuri, K.
\newblock Data poisoning attacks against online learning.
\newblock \emph{arXiv preprint arXiv:1808.08994}, 2018.

\bibitem[Xiao et~al.(2015)Xiao, Biggio, Brown, Fumera, Eckert, and
  Roli]{xiao2015feature}
Xiao, H., Biggio, B., Brown, G., Fumera, G., Eckert, C., and Roli, F.
\newblock Is feature selection secure against training data poisoning?
\newblock In \emph{International Conference on Machine Learning}, pp.\
  1689--1698, 2015.

\end{thebibliography}
\bibliographystyle{icml2019}
\appendix
\onecolumn
\section{Details on the offline attacks}

\subsection{Proof of Theorem \ref{thm:greedy}}\label{app:greedy}
\begin{proof}
The optimization problem $P_1$ is a quadratic program with linear constraints in $\{\bepsilon_a\}_{a\in\mathcal{A}}$. Now it remains to show that the constraint set is non-empty.

Given any reward instance $\{\by_a\}_{a\in\mathcal{A}}$, any margin parameter $\xi>0$ and any $\bepsilon_{a^*}$, one can check that
\begin{equation}
\bepsilon_a = \left[(\by_{a^*}+\bepsilon_{a^*})^T\mathds{1}/m_{a^*}  - \by_{a}^T\mathds{1}/m_{a}-\xi\right]\mathds{1},~~~~\forall a\not = a^*,
\end{equation}
satisfies the constraints of problem $P_1$. That is the constraint set of problem $P_1$ is non-empty.

Thus, there exists at least one optimal solution of problem $P_1$ since $P_1$ is a quadratic program with non-empty and compact constraints. The result follows from Proposition~\ref{prop:offline}.
\end{proof}

\subsection{Details on attacking Thompson Sampling}
\begin{lemma}\label{lem:conv}
Given some constants $C_i>0$ for any $i<K$. The function $f(\bx) = \sum_{i=1}^{K-1}\Phi(C_i x_i - C_i x_K)$ is convex on the domain $D=\{\bx\in\mathcal{R}^K|x_i-x_K\leq0, \forall i <K\}$.
\end{lemma}
\begin{proof}
We prove the result by checking the Hessian matrix $H$ of function $f(\bx)$. Note that $\Phi(x)$ is the cumulative distribution function of the standard normal distribution $\mathcal{N}(0,1)$. For any $i<K$, we have that 
\begin{align}
\frac{\partial f}{\partial x_i} &= \frac{C_i}{\sqrt{2\pi}}e^{-(C_i x_i - C_i x_K)^2/2},\\
\frac{\partial^2 f}{\partial x_i^2} &= -\frac{C_i^2}{\sqrt{2\pi}}e^{-(C_i x_i - C_i x_K)^2/2}(C_i x_i - C_i x_K).
\end{align}
On the other hand, we have that
\begin{align}
\frac{\partial f}{\partial x_K} &= \sum_{i=1}^{K-1}-\frac{C_i}{\sqrt{2\pi}}e^{-(C_i x_i - C_i x_K)^2/2} = \sum_{i=1}^{K-1}-\frac{\partial f}{\partial x_i},\\
\frac{\partial^2 f}{\partial x_K^2} &= \sum_{i=1}^{K-1}-\frac{C_i^2}{\sqrt{2\pi}}e^{-(C_i x_i - C_i x_K)^2/2}(C_i x_i - C_i x_K) = \sum_{i=1}^{K-1}\frac{\partial^2 f}{\partial x_i^2}.
\end{align}
Now, we derive the other coefficients. For any pair $(i,j)$ such that $i\not = j$, $i<K$ and $j<K$, we have that
\begin{equation}
\frac{\partial^2 f}{\partial x_i\partial x_j} = 0.
\end{equation}
For any $i<K$, we have that
\begin{align}
\frac{\partial^2 f}{\partial x_i\partial x_K} & = \frac{C_i^2}{\sqrt{2\pi}}e^{-(C_i x_i - C_i x_K)^2/2}(C_i x_i - C_i x_K) = - \frac{\partial^2 f}{\partial x_i^2},\\
\frac{\partial^2 f}{\partial x_K\partial x_i} &= - \frac{\partial^2 f}{\partial x_i^2}
\end{align}
Since the constants $C_i$ are positive, we have that $\frac{\partial^2 f}{\partial x_i^2} \geq0 $ in the domain $D$. The Hessian matrix of $f$ is the following,
\begin{equation}
H= 
\begin{bmatrix}
\frac{\partial^2 f}{\partial x_1^2} & 0  &\dots & 0 &-\frac{\partial^2 f}{\partial x_1^2}\\
0 & \frac{\partial^2 f}{\partial x_2^2}  &\dots & 0 & -\frac{\partial^2 f}{\partial x_2^2}\\
\vdots & \vdots & \ddots &\vdots & \vdots \\
0 & 0 & \dots & \frac{\partial^2 f}{\partial x_{K-1}^2} &-\frac{\partial^2 f}{\partial x_{K-1}^2}\\
-\frac{\partial^2 f}{\partial x_1^2} & -\frac{\partial^2 f}{\partial x_2^2} &\dots & -\frac{\partial^2 f}{\partial x_{K-1}^2}  & \sum_{i=1}^{K-1} \frac{\partial^2 f}{\partial x_i^2}
\end{bmatrix}.
\end{equation}
Hence, for any vector $\by\in\mathcal{R}^K$, we have that
\begin{equation}
\by^TH\by = \by^T
\begin{bmatrix}
\frac{\partial^2 f}{\partial x_1^2}(y_1-y_K)\\
\frac{\partial^2 f}{\partial x_2^2}(y_2-y_K)\\
\vdots\\
\frac{\partial^2 f}{\partial x_{K-1}^2}(y_{K-1}-y_K)\\
\sum_{i=1}^{K-1}-\frac{\partial^2 f}{\partial x_i^2}(y_i-y_K)\\
\end{bmatrix}
=\sum_{i=1}^{K-1}\frac{\partial^2 f}{\partial x_i^2}(y_i-y_K)^2 \geq 0.
\end{equation}
Since $H$ is positive semi-definite, we show that $f(\bx)$ is convex on the domain $D$.
\end{proof}

\subsection{Proof of Proposition~\ref{prop:convex}}\label{app:convex}
\begin{proof}
By Lemma~\ref{lem:conv} and the fact that affine mapping keeps the convexity, we have the result.
\end{proof}

\subsection{Another relaxation of P for Thompson Sampling}\label{app:P4}
We may find a sufficient constraint to equation (\ref{eqn:unionTS}) as 
\begin{equation}
\Phi\left(\frac{\tilde{\mu}_a(T)-\tilde{\mu}_{a^*}(T)}{\sigma^3\sqrt{1/m_a+1/m_{a^*}}}\right)\leq\frac{\delta}{K-1},~~~~\forall a\not=a^*.
\end{equation}
Then, we derive another relaxation of P as 
\begin{align}
P_4: \min_{\bepsilon_a:a\in\mathcal{A}} &~~~~ \sum_{a\in\mathcal{A}}||\bepsilon_a||^2_2\\
s.t. &~~~~\tilde{\mu}_a(T)-\tilde{\mu}_{a^*}(T)\leq \sigma^3\sqrt{1/m_a+1/m_{a^*}} \Phi^{-1}\left(\frac{\delta}{K-1}\right), ~~~~\forall a\not = a^*
\end{align}
Note that problem $P_4$ is a quadratic program with linear constraints.

\section{Details on the online attacks}
\subsection{Proof of Proposition~\ref{prop:constant}}\label{app:constant}
\begin{proof}
By equation (\ref{eqn:regretDecompose}), a logarithmic regret bound implies that the bandit algorithm satisfies $\mathbb{E}[N_a(T)] = O(\log T)$ for any suboptimal arm $a$. Note that the oracle constant attack shifts the expected rewards of all arms except for the target arm $a^*$. Since $C_a > [\mu_{a}-\mu_{a^*}]^+$, $\forall a \not= a^*$, the best arm is now the target arm $a^*$. Then, the bandit algorithm satisfies $\mathbb{E}[N_a(T)] = O(\log T)$, $\forall a \not= a^*$. Thus, the expected number of pulling the target arm is
\begin{equation}
\mathbb{E}[N_{a^*}(T)] = T-\sum_{a\not= a^*}\mathbb{E}[N_a(T)] = T-o(T).
\end{equation}
Since the attacker does not attack the target arm, we have that
\begin{equation}
\mathbb{E}[C(T)] = \mathbb{E}\left[\sum_{t=1}^T|\epsilon_t|\right] = \sum_{a\not = a^*} C_a \mathbb{E}[N_a(T)] = O\left(\sum_{a\not= a^*}C_a\log T\right).
\end{equation}

On the other hand, suppose there exists an arm $i\not= a^*$ such that $C_i \leq [\mu_{i}-\mu_{a^*}]^+$, then the attack is not successful. In the case that $C_i < [\mu_{i}-\mu_{a^*}]^+$, the arm $i$ is the best arm rather than the target arm $a^*$ in the shifted bandit problem. That is the expected number of pulling arm $a^*$ is $\mathbb{E}[N_{a^*}(T)] = O(\log T)$. In the case that $C_i = [\mu_{i}-\mu_{a^*}]^+$, the arm $i$ and $a^*$ are both the best arms. That is the expected attack cost is $\mathbb{E}[C(T)] = T-o(T)$. In neither case is the attack successful. This concludes the proof.
\end{proof}

\subsection{Proof of Theorem~\ref{thm:adaptive}}\label{app:adaptive}
\begin{proof}
Given any $\delta >0$, we have that $\mathbb{P}(E) > 1-\delta$ by Lemma~\ref{lem:concen}.
Under the event $E$, we have that at any time $t$ and for any arm $a\not=a^*$,
\begin{align}
\mu_a - \mu_{a^*} & < \hat{\mu}_a(t) - \mu_{a^*} + \beta(N_a(t))\\
&< \hat{\mu}_a(t) - \hat{\mu}_{a^*}(t) + \beta(N_a(t))+\beta(N_{a^*}(t)),
\end{align}
which implies that
\begin{equation}
[\mu_a - \mu_{a^*}]^+ < [\hat{\mu}_a(t) - \hat{\mu}_{a^*}(t) + \beta(N_a(t))+\beta(N_{a^*}(t))]^+.
\end{equation}
By the same argument in the proof of Proposition~\ref{prop:constant}, we have that under event $E$, the attacker is taking an effective attack for any bandit algorithm. 

Recall that the bandit algorithm has a high-probability bound such that the regret is bounded by $O(\log T)$ with probability at least $1-\delta$. Under event $E$, we have that $N_a(T) = O(\log T)$ for any $a\not=a^*$ with high probability. Thus, with probability at least $1-2\delta$, we have that $N_{a^*}(T) = T-o(T)$. It remains to bound the cost of the attacker, i.e., $\sum_{t}|\epsilon_t|$.

Given any arm $a\not=a^*$, any time $t$ and under the event $E$, we have that
\begin{align}
\hat{\mu}_a(t)-\hat{\mu}_{a^*}(t) &<\mu_a -\hat{\mu}_{a^*}(t) + \beta(N_a(t))\\
&<\mu_a -\mu_{a^*} + \beta(N_a(t)) + \beta(N_{a^*}(t)).
\end{align}
This implies that
\begin{align}
[\hat{\mu}_a(t)-\hat{\mu}_{a^*}(t)+&\beta(N_a(t)) + \beta(N_{a^*}(t))]^+ \\
&<[\mu_a -\mu_{a^*} + 2\beta(N_a(t)) + 2\beta(N_{a^*}(t))]^+\\
&\leq [\mu_a -\mu_{a^*}]^+ + 2\beta(N_a(t)) + 2\beta(N_{a^*}(t)).
\end{align}
Thus, the first statement follows. By the fact that $\beta(n)$ is a decreasing function, we have that
\begin{align}
\sum_{t=1}^T|\epsilon_t| & \leq \sum_{t=1}^T\left([\mu_{a_t}-\mu_{a^*}]^+ + 4\beta(1)\right) \mathds{1}\{a_t \not= a^*\}\\
& =\sum_{a\not = a^*}\left([\mu_{a}-\mu_{a^*}]^+ + 4\beta(1)\right)N_a(T) \\
&\leq O\left(\sum_{a\not= a^*}\left([\mu_a-\mu_{a^*}]^+ + 4\beta(1)\right) \log T\right).
\end{align}
\end{proof}

\end{document}